\documentclass[journal]{IEEEtran}

\usepackage{times}
\usepackage{epsfig}
\usepackage{graphicx}
\usepackage{amsmath}
\usepackage{amssymb}
\usepackage{booktabs}
\usepackage{enumitem}
\usepackage[ruled,vlined]{algorithm2e}
\usepackage{amsthm}
\newtheorem{theorem}{Theorem}
\newtheorem{definition}{Definition}[section]
\usepackage{amsfonts}
\newtheorem{lemma}[theorem]{Lemma}

\usepackage{multicol}
\usepackage{color, colortbl}
\usepackage[table,xcdraw]{xcolor}
\usepackage{comment}
\usepackage{bm}

\usepackage[pagebackref=true,breaklinks=true,colorlinks,bookmarks=false]{hyperref}

\usepackage{multirow}

\usepackage{enumitem}

\begin{document}
%
\title{Cooperative Reinforcement Learning on Traffic Signal Control}
%
%
%

\author{Chi-Chun Chao,~\IEEEmembership{Student Member,~IEEE,}
        , Jun-Wei Hsieh,~\IEEEmembership{Member,~IEEE,}
        and Bor-Shiun Wang,~\IEEEmembership{Student Member,~IEEE,}
 }       

%
%

\markboth{Journal of \LaTeX\ Class Files,~Vol.~XX, No.~XX, May~2021}%
{Shell \MakeLowercase{\textit{et al.}}: Bare Demo of IEEEtran.cls for IEEE Journals}
%



\maketitle

\begin{abstract}
Traffic signal control is a challenging real-world problem aiming to minimize overall travel time by coordinating vehicle movements at road intersections. Existing traffic signal control systems in use still rely heavily on oversimplified information and rule-based methods. Specifically, the periodicity of green/red light alternations can be considered as a prior for better planning of each agent in policy optimization. To better learn such adaptive and predictive priors, traditional 
 RL-based methods can only return a fixed length from predefined action pool with only local agents. If there is no cooperation between these agents, some agents often make conflicts to other agents and thus decrease the whole throughput.  This paper proposes a cooperative, multi-objective architecture with age-decaying weights to better estimate multiple reward terms for traffic signal control optimization, which termed COoperative Multi-Objective Multi-Agent Deep Deterministic Policy Gradient (COMMA-DDPG). Two types of agents running to maximize rewards of different goals - one for local traffic optimization at each intersection and the other for global traffic waiting time optimization.  The global agent is used to guide the local agents as a means for aiding faster learning but not used in the inference phase. We also provide an analysis of solution existence together with convergence proof for the proposed RL optimization. Evaluation is performed using real-world traffic data collected using traffic cameras from an Asian country. Our method can effectively reduce the total delayed time by 60\%. Results demonstrate its superiority when compared to SoTA methods.
\end{abstract}



%
\IEEEpeerreviewmaketitle

\section{Introduction}
\label{sec:intro}
%
%
%
%

\IEEEPARstart{T}{raffic} 
signal control is a challenging real-world problem whose goal tries to minimize the overall vehicle travel time by coordinating the traffic movements at road intersections. Existing traffic signal control systems in use still rely heavily on manually designed rules which cannot adapt to dynamic traffic changes. Recent advance in reinforcement learning (RL), especially deep RL~\cite{alemzadeh2020adaptive,zheng2019diagnosing}, offers excellent capability to work with high dimensional data, where agents can learn a state abstraction and policy approximation directly from input states.  This paper explores the possibility of RL to on-policy traffic signal control with fewer assumptions.

In literature, there have been different RL-based frameworks~\cite{wei2021recent}  proposed for traffic signal control.  Most of them~\cite{zheng2019diagnosing,mannion2016experimental,pham2013learning,van2016coordinated,wei2018intellilight,arel2010reinforcement,calvo2018heterogeneous} are value-based and can achieve faster convergences on traffic signal control.  However, the actions, states, and time space they can handle are discrete. Thus, the time slots for each action to be executed are fixed and cannot reflect the real requirements to optimize traffic conditions.  Moreover, a small change in the value function will cause great effects on the policy decision.  To make decisions on continuous space, recent policy-based RL methods~\cite{chu2019multi,nishi2018traffic,mousavi2017traffic} become more popularly adopted in traffic signal control so that a non-discrete length of phase duration can be inferred.   However, its gradient estimation is strongly dependent on sampling and not stable if sampling cases are not general, and thus easily trapped to a non-optimal solution. 

Another vital problem of the above RL-based methods is that their agents are trained in an off-policy way.  They are not retrained on-fly during inference.   Its means their policy decision strategy cannot be amended and adapted to real traffic conditions.  Moreover, they make action decisions continuously along the time and  give drivers very short  reaction time to change their behaviors.  Outcomes from these methods are less practical since in real-world traffic control scenarios, considering the next traffic light phase from pre-defined discrete cyclic sequences of red/green lights is actually important to let drivers know how much reminding time the next traffic signal phase will be changed.  For traffic optimization, this means when the agent makes decisions not only on which action to be performed but also how long it should be performed.  To determine a proper period for action to be executed, some frameworks~\cite{aslani2017adaptive,aslani2018traffic} pre-define some time slots for the agent to choose for computation efficiency and traffic control simplification.  However, this  solution of pre-defining time slots is less flexible than an on-demand solution to better relieve traffic congestion.

To bridge the gaps between value-based and policy-based RL approaches, the actor-critic framework is widely adopted to stabilize the RL training process, where the policy structure is known as the actor and the estimated value function is known as the critic.  Thus, there are some actor-critic frameworks proposed for traffic signal control.  For example, in ~\cite{pang2019deep,wu2020control}, a model-free actor-critic method named as ``Deep Deterministic Policy Gradient'' (DDPG)~\cite{lillicrap2015continuous} was adopted to learn a deterministic policy mapping states to actions.  However, it is a ``single-agent'' solution and cannot output a proper execution period for the chosen action to more effectively relieve traffic congestion.  This paper incorporates multiple agents in an actor-critic framework to develop a COoperative  Multi-objective Multi-Agent DDPG (COMMA-DDPG) for optimal traffic signal control. This novelty of our method is to introduce a global agent to trade off different local agents' requirements and cooperate them to find better strategies for traffic signal control.  The global agent is used as a means for aiding faster learning and not used in the inference phase.  This idea is very different from other RL-based multi-agent methods which use only local agents to search solutions and often produce conflicts between two agents' strategies, thus decreasing the whole throughput.

Fig. \ref{fig:comma} shows the diagram of this COMMA-DDPG architecture.  A local agent is first used to learn and optimize the policy at an intersection. During the training process, we introduce a global agent to cooperate each local agent and then optimize  the global traffic throughput among all intersections in the whole observed site.
 With the actor-critic framework, the global agent can optimize and send various information exchange among local intersections to local agents so as to optimize the final reward globally.   It can aid faster learning while not constraining the deployment since it is not used during the inference phase. To finish this goal, the parameters of each local DDPG-based agent is initialized by the global agent.  Thus, the COMMA-DDPG framework can select the best policy for controlling periodical phases of traffic signals that maximizes throughput by trading off the requirements and reducing conflicts between agents.  It can yield a dynamic length of the next traffic light phase in seconds. This is very different from other RL-based methods which can only return a fixed length from a predefined action pool. Then, the remaining seconds of a traffic light phase can be dynamically predicted and sent to the drivers for doing next driving plans.  It is noticed again that the global agent is only used during the training process to cooperate different local agents' requirements.  

{\bf Convergence Analysis.}
To prove the convergence of COMMA-DDPG, we also analyze the existence and uniqueness of our actor-critic model in the ``Appedix'' Section for proof. This proof provides theoretical supports of the convergence of our COMMA-DDPG approach. 

{\bf Evaluation.} 
 Our traffic data consists of visual traffic monitoring sequences from five consecutive intersections during morning rush hour in one local country of Taiwan. We conducted various ablation studies on COMMA-DDPG with different SoTA to evaluate the performance comparisons. Results show that COMMA-DDPG provides significant improvement of the overall traffic waiting time in better alleviating traffic congestion.

The remaining of this paper is organized in the following. Section 2 surveys related works of intelligent traffic control. Section 3 describes the RL notations and schemes. Section 4 describes the architecture of COMMA-DDPG. Section 5 shows the experiment results.
\begin{figure}
    \centering
    \includegraphics[scale=0.3]{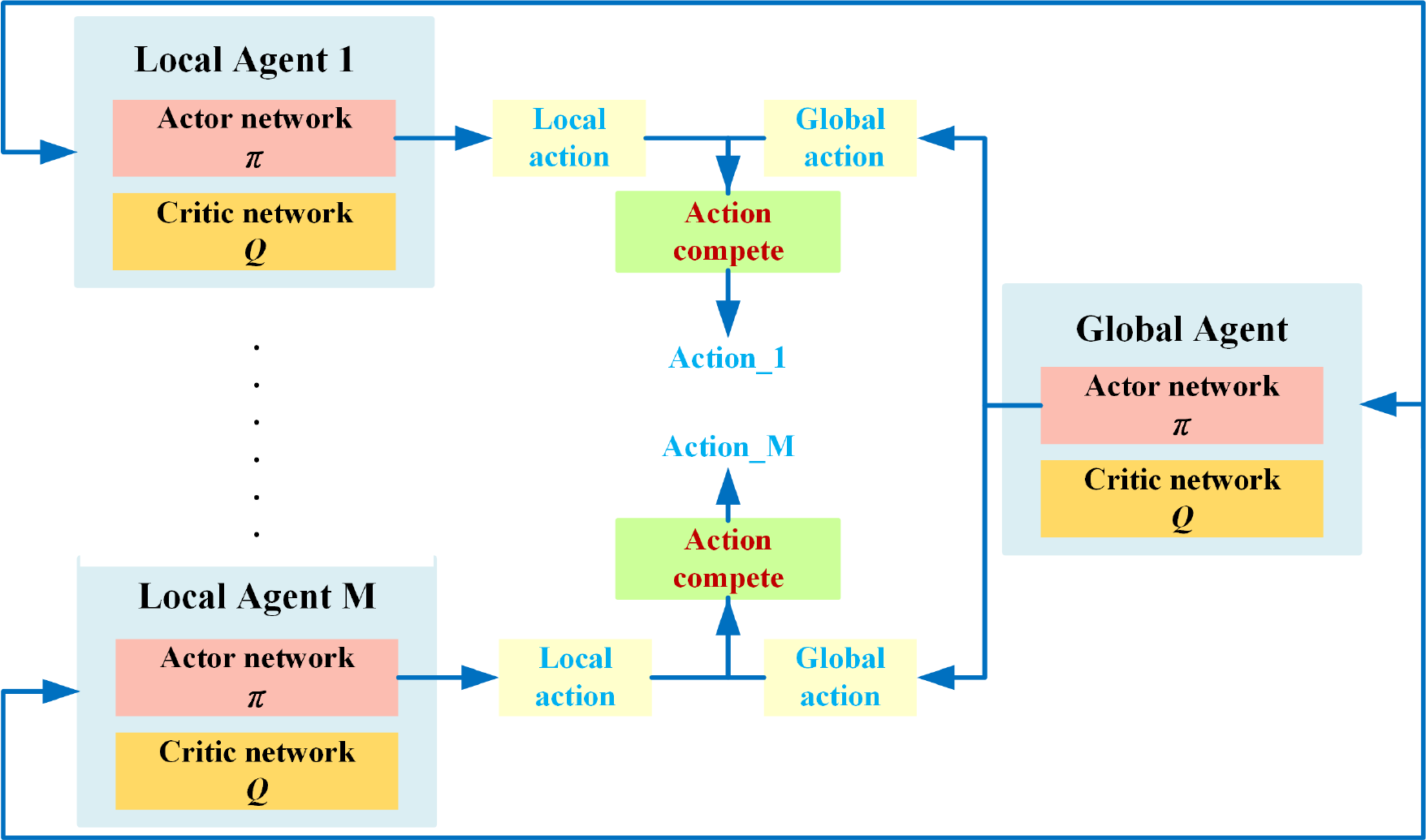}
    \caption{COMMA-DDPG architecture.}
    \label{fig:comma}
\end{figure}

\section{Related Work}
\label{sec:related}

\subsection{Traditional traffic control}
\label{related:Traditionaltrafficcontrol}

Traditional traffic control methods can be categorized into the following classes: (1) fixed time control~\cite{roess2004traffic}, (2) actuated control~\cite{fellendorf1994vissim,mirchandani2001real}, and (3) adaptive control~\cite{zheng2019learning,zheng2019diagnosing}. They are mostly based on  human knowledge or manual efforts to design an appropriate cycle length and strategies for better traffic control. The involved manual tasks will make parameter settings very cumbersome and difficult to satisfy different scenarios' requirements including peak hours, normal hours, and off-peak hours. Fixed time control is simply, easy, and thus becomes the most commonly adopted method in traffic signal control. Its time length for each traffic light phase is pre-calculated and still keep the same even though traffic conditions have been changed. Actuated control determines traffic conditions using pre-determined thresholds; that is, if the traffic condition (e.g., the car queue length) exceeds a threshold, a green light will be issued accordingly. Adaptive control methods including SCATS~\cite{lowrie1990scats} and SCOOT~\cite{hunt1981scoot}  determine the best signal phase according to the on-going traffic conditions, and thus can achieve more effective traffic optimization.

\subsection{RL traffic control}
\label{related:RLtrafficcontrol}
The recent advancement of RL shed a light on automatic traffic control improvement. RL agents use knowledge from the traffic data to learn and optimize the policies without human intervention for traffic sign control. There are two main approaches to solving traffic sign control problems, $i.e.$, value-based and policy-based. There is also a hybrid, actor-critic approach, which employs both value-based and policy-based searches. The value-based method first estimates the value (expected return) of being in a given state and then finds the best policy from the estimated value function.  One of the most used value-based methods is $Q$ learning~\cite{watkins1992q}.  The first $Q$-learning method applied to control traffic signals at street intersection is traced from~\cite{abdoos2011traffic}. However, in $Q$ learning, a huge table should be created and updated to store the $Q$ values of each action in each state. Thus, it is both memory- and time-consuming, and improper for problems with complicated states and actions.  Recently, the advent of deep learning has cast significant impacts on many areas such as object detection, speech recognition, language translation, and so on. Thus, deep reinforcement learning methods such as value-based methods DQN (Deep Q Network)~\cite{zheng2019diagnosing,wei2019presslight,wei2019colight} and AC (Actor Critic) methods~\cite{aslani2017adaptive,xiong2019learning} are widely used in traffic control.  However, to the best of our knowledge, none of the above methods can reliably predict the remaining seconds of a traffic light phase. We note that such predicting capability can provide useful information for drivers to prepare their next driving behaviors in real-life traffic control regarding passenger safety. The proposed MOMA-DDPG model in this paper is a new action design that can solve and predict this phase remaining time.

On the other hand, RL can be classified according to the adopted action schemes, such as: (1) setting the length of green light, (2) choosing whether to change phase, and (3) choosing the next phase.
The DQN and AC methods are suitable for action schema 2 and 3, but improper for setting the length of the green line since the action space of DQN is discrete and huge for this task, and much calculation efficacy will be wasted for the AC method. Compared with DQN whose action space is discrete, DDPG~\cite{lillicrap2015continuous} can solve continuous action spaces which are more suitable to model the length of green light.   Although DDPG is originated from the AC method, it performs more robustly than the AC method since it creates two networks as regression models to estimate values.  Therefore, this paper uses DDPG~\cite{lillicrap2015continuous} as the main architecture to adapt and model the action space of green light length. Given a range of green light duration, the proposed DDPG-based method can easily output seconds within the range. In the past, the DDPG-based traffic control frameworks ~\cite{pang2019deep,wu2020control} focus on only a single intersection. This paper will use the idea of DDPG to model traffic conditions on multiple intersections by introducing a global agent to trade off different local agents’requirements  and then find  better  strategies  for traffic  signal  control.  

\section{RL Background and Notations}
\label{sec:background}

The basic elements of a RL problem for traffic signal control can be formulated as the Markov Decision Process (MDP) mathematical framework of $< S,A,T,R,\gamma >$, with the following definitions:
\begin{itemize}[leftmargin=16pt]
\setlength\itemsep{-0.028cm}
\item $S$ denotes the set of states, which is the set of all lanes containing all possible vehicles. $s_t  \in S$ is a state at time step $t$ for an agent.
\item $A$ denotes the set of possible actions, which is the duration of green light. In our scenarios, both duration lengths for a traffic cycle and a yellow light are fixed.  Then, once the state of green light is chosen, the duration of a red light can be determined.  At time step $t$, the agent can take an action $a_t$ from  $A$.
\item $T$ denotes the transition function, which stores the probability of an agent transiting from state $s_t$ (at time $t$) to $s_{t+1}$ (at time $t+1$) if the action $a_t$ is taken; that is, $T(s_{t+1}|s_{t},a_{t} ):S \times A \rightarrow S.$
\item 	$R$ denotes the reward, where at time step $t$, the agent obtains a reward $r_t$ specified by a reward function $R(s_{t},a_{t} )$ if the action $a_t$ is taken under state $s_t$. 
\item $\gamma$ denotes the discount, which not only controls the importance of the immediate reward versus future rewards, but also ensures the convergence of the value function, where, $\gamma \in [0,1)$.
\end{itemize}

At time-step $t$, the agent determines its next action $a_t$ based on the current state $s_t$.  After executing $a_t$, it will be transited to next state $s_{t+1}$ and receive a reward $r_{t} (s,a)$; that is, $r_{t} (s,a)= \mathbb{E}[R_{t}|s_{t}=s,a_{t}=a]$,  where $R_t$ is named as the one-step reward. The way that the RL agent chooses an action is named policy and denoted by $\pi$. Policy is a function $\pi (s)$ that chooses an action from the current state $s$; that is, $\pi (s):S \rightarrow A$. The goal of this paper is to find such a policy to maximize the future reward $G_t$:
\begin{equation}
\label{equ:Gt}
    G_{t}=\Sigma_{k=0}^{\infty}\gamma^{k}R_{t+k}.
\end{equation}

A value function $V(s_t)$ indicates how good the agent is at state $s_t$, and is defined as
the expected total return of the agent starting from $s_t$. If $V(s_t)$ is conditioned on a given strategy $\pi$, it will be expressed by $V^{\pi}(s_{t})$; that is, $V^{\pi}(s_{t})=\mathbb{E}[G_{t}|s_{t}=s], \forall s_{t} \in S$.  The optimal policy $\pi^{*}$ at state $s_t$ can be found by solving  
\begin{equation}
    \pi^{*}(s_t)=arg \max\limits_{\pi}V^{\pi}(s_t), 
\label{eq2}
\end{equation}  
where $V^{\pi}(s_t)$ is the state-value function for a policy $\pi$.  Similarly, we can define the expected return of taking action $a$ in state $s_t$ under a policy $\pi$ denoted by a $Q$ function: 
\begin{equation}
    Q^{\pi}(s_{t},a_{t})=\mathbb{E}[G_{t}|s_{t}=s,a_{t}=a].
\end{equation}
The relationship between $Q^{\pi}(s_t, a_t)$ and $V^{\pi} (s_t)$ is derived as
\begin{equation}
   {{\rm{V}}^\pi }{\rm{(}}s) = \sum\limits_{a \in A} {\pi (a|s){Q^\pi }(s,a)}. 
\end{equation}
Then, the optimal solution $Q^{*}(s_{t},a)$ is found by iteratively solving:
\begin{equation}
Q^{*}(s_{t},a)=\max\limits_{\pi}Q^{\pi}(s_{t},a).
\end{equation} 
With the $Q$ function, the optimal policy $\pi^{*}$ at state $s_t$ can be found by solving:  
\begin{equation}
    \pi^{*}(s_t)=arg \max\limits_{a}Q^{*}(s_t,a). 
\end{equation}  
$Q^{*}(s_{t},a)$ is the sum of two terms: (i) the instant reward after a period of execution in the state $s_t$ and (ii) the discount expected future reward after the transition to the next state $s_{t+1}$. Then, we can use the Bellman equation \cite{bellmanequation} to express $Q^{*}(s_t,a)$ as follows: 
\begin{equation}
    Q^{*}(s_{t},a)=R(s_{t},a)+\gamma\mathbb{E}_{s_{t+1}}[V^{*}(s_{t+1})]. 
\label{eq:BellmanQ}
\end{equation}  
$V^*(s_t)$ is the maximum expected total reward from state $s_t$ to the end.  It will be the maximum value of $Q^*(s,a)$ among all possible actions. Then, $V^*$ can be obtained from $Q^*$ as follows:
\begin{equation}
{V^*}\left( {{s_t}} \right) = \mathop {\max }\limits_a {Q^*}\left( {{s_t},a} \right),\forall {s_t} \in S.
\end{equation} 
Two strategies, $i.e$., value iteration and policy iteration can be used to calculate the optimal value function $V^*(s_t)$.  The value iteration calculates the optimal state value function by iteratively improving the estimation of $V(s)$. It repeatedly updates the values of $Q(s,a)$ and $V(s)$ until they converge.  Instead of repeatedly improving the estimation of $V(s)$, the policy iteration redefines the policy at each step and calculate the value according to this new policy until the policy converges.

{\bf Deep Q-Network (DQN):} In \cite{hester2018deep,van2016deep}, a deep neural network is used to approximate the $Q$ function, which enables the RL algorithm to learn $Q$ well in high-dimensional spaces. Let $Q_{tar}$ be the targeted true value which is expressed as $Q_{tar}= r+\gamma \max\limits_{a'}Q(s',a';\theta)$.  In addition, let $Q(s,a;\theta)$ be the estimated value, where $\theta$ is the set of parameters of the used deep neural network.   We define the loss function for training the DQN as:
\begin{equation}
L(\theta)=\mathbb{E}_{s,a,r,s'}[(Q_{tar}-Q(s,a;\theta))^{2}].
\end{equation}
As described in \cite{mnih2015human}, the value of $Q_{tar}$ is constantly changing and often overestimated during training and results in the problem of unstable convergence of the $Q$ function.  
In ~\cite{van2016deep}, a DDQN (Double DQNs) was proposed to deal with this unstable problem by separating the neural network into two DQNs with two value functions such that there are two sets of weights $\theta$ and $\phi$ for parameterizing the original value function and the second target network, respectively.  The second DQN $Q_{tar}$ with parameters $\phi$ is a lagged copy of the first DQN $Q(s,q;\theta)$ to fairly evaluate the $Q$ value; that is,
\begin{equation}
Q_{tar} = r+\gamma Q(s',\max\limits_{a'}Q(s',a';\theta);\phi).
\end{equation}

{\bf Deep Deterministic Policy Gradient (DDPG)}:
DDPG is also a type of model-free and off-policy, and it also uses a deep neural network for function approximation. But unlike DQN which can only solve discrete and low-dimensional action spaces, DDPG can solve continuous action spaces. In addition, DQN is a value-based method, while DDPG is an Actor-Critic method, which has both a value function network (critic) and a policy network (actor). The critic network used in DDPG is the same as the actor-critic network(described before). The difference between DDPG and the actor-critic network is that:  derived from DDQN~\cite{van2016deep}, DDPG makes the training process more robustly by creating two DQNs (target and now) to estimate the value functions.

\section{Method}
\label{sec:method}
This paper proposes a cooperative, multi-objective architecture with age-decaying weights for traffic signal control optimization.  This architecture represents each intersection with a DDPG architecture which contains a critic network and an actor network.  In a given state, the critic network is responsible for judging the value of doing an action, and the actor network tries to make the best decision to output an action. The outcome of an action is the number of seconds of green light.  This paper assumes that the duration for a phase cycle (green, yellow, red) is different at different intersections but fixed at an intersection.  In addition, the duration of yellow light is the same and fixed for all intersections. Then, the duration of red light can be directly derived once the duration of green is known.     

The RL-based method for traffic signal control can be value-based or policy-based. The value-based method can achieve faster convergence on traffic signal control but its time space is discrete, and cannot reflect the real requirements to optimize traffic conditions. The policy-based method can infer a non-discrete length of phase duration but its gradient estimation is strongly dependent on sampling an not stable due to sampling bias.  Thus, the DDPG method is adopted in this paper to concurrently learns the desired $Q$-function and  the corresponding policy.

 \begin{figure}[t]
     \centering
     (a)\includegraphics[scale=0.28]{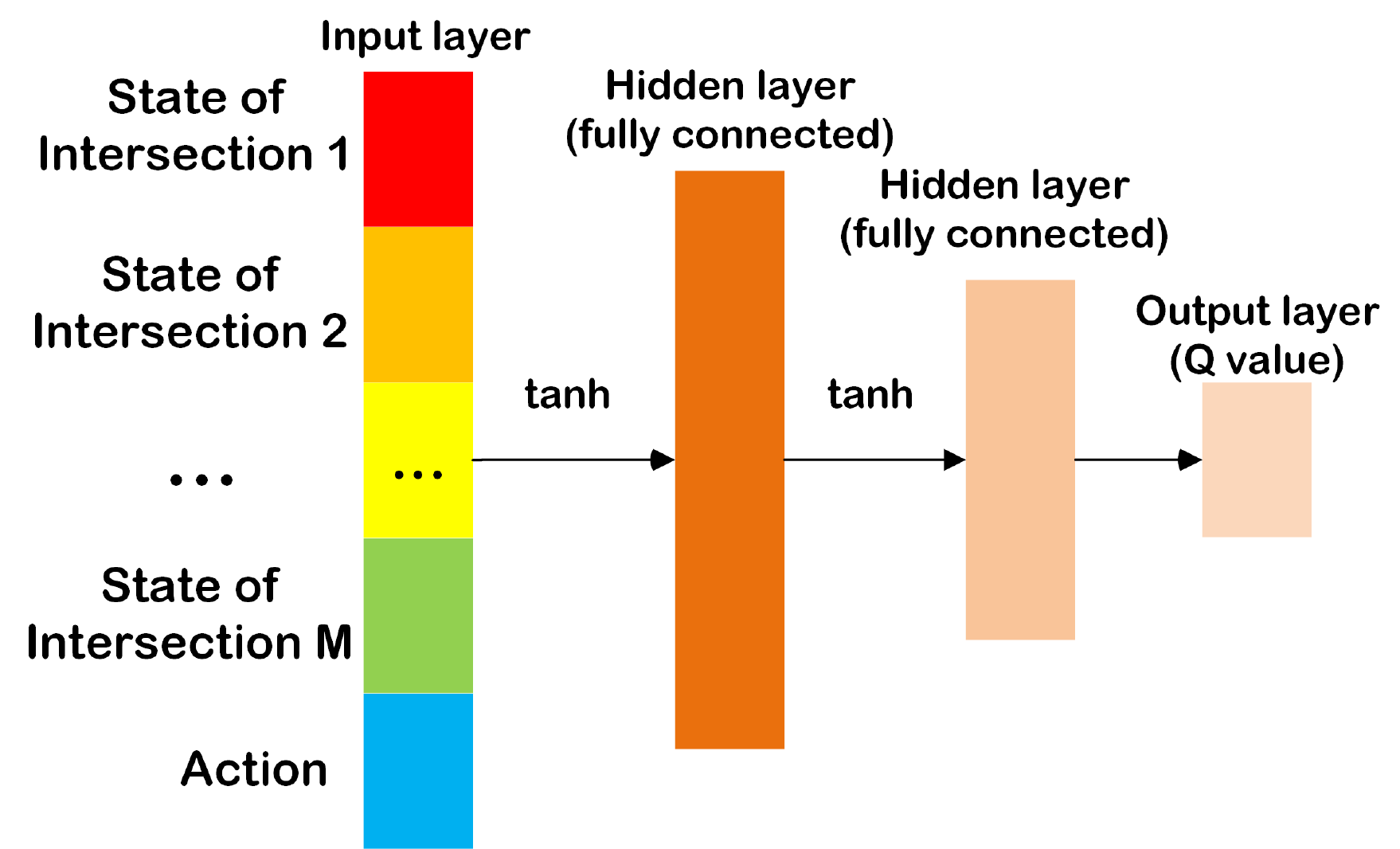}
     \centering
     (b)\includegraphics[scale=0.28]{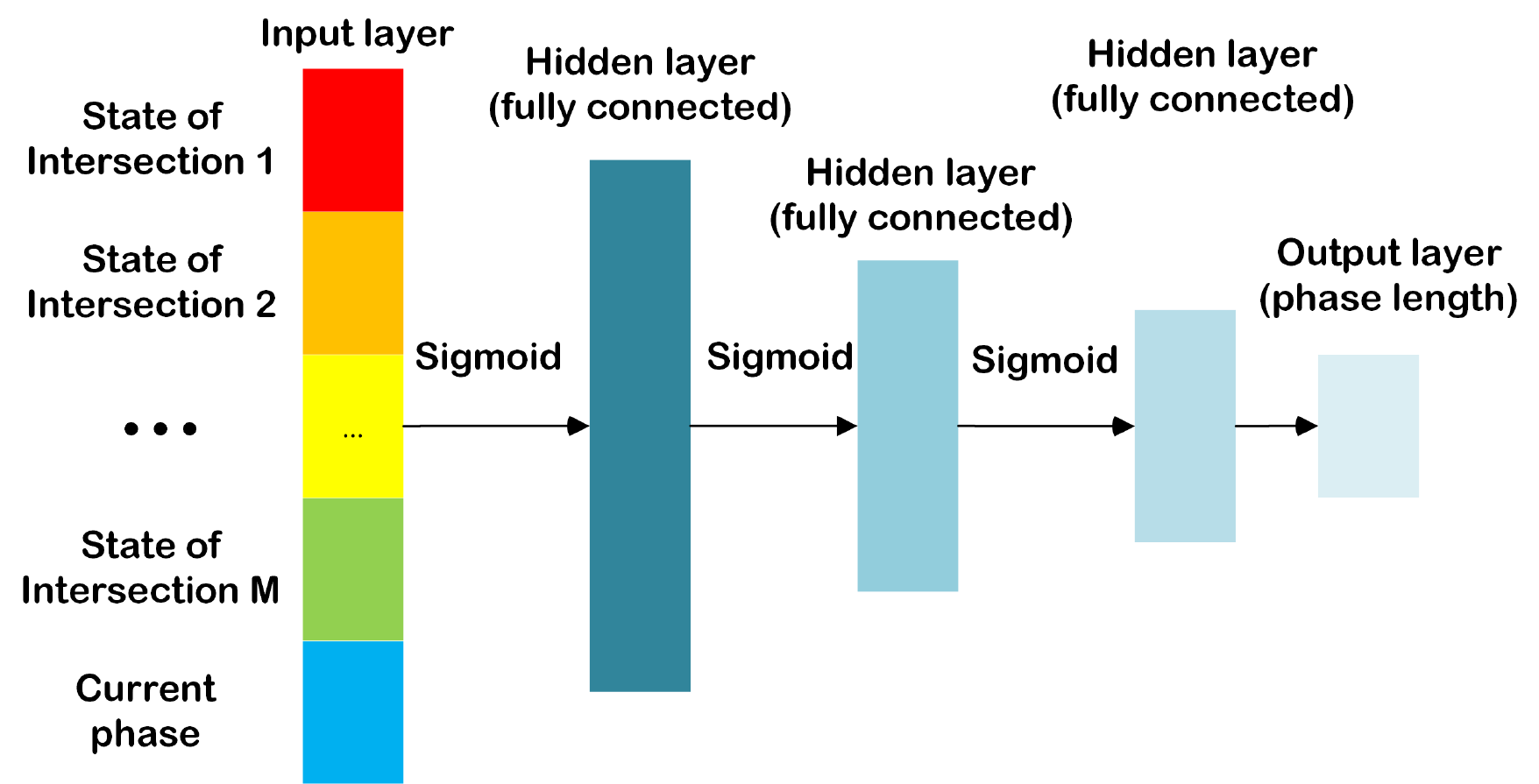}
     \caption{Architectures for local agent. (a) Local critic.(b) Local actor.}
     \label{fig:local_agent}
 \end{figure}

 \begin{figure}[t]
     \centering
     (a)\includegraphics[scale=0.28]{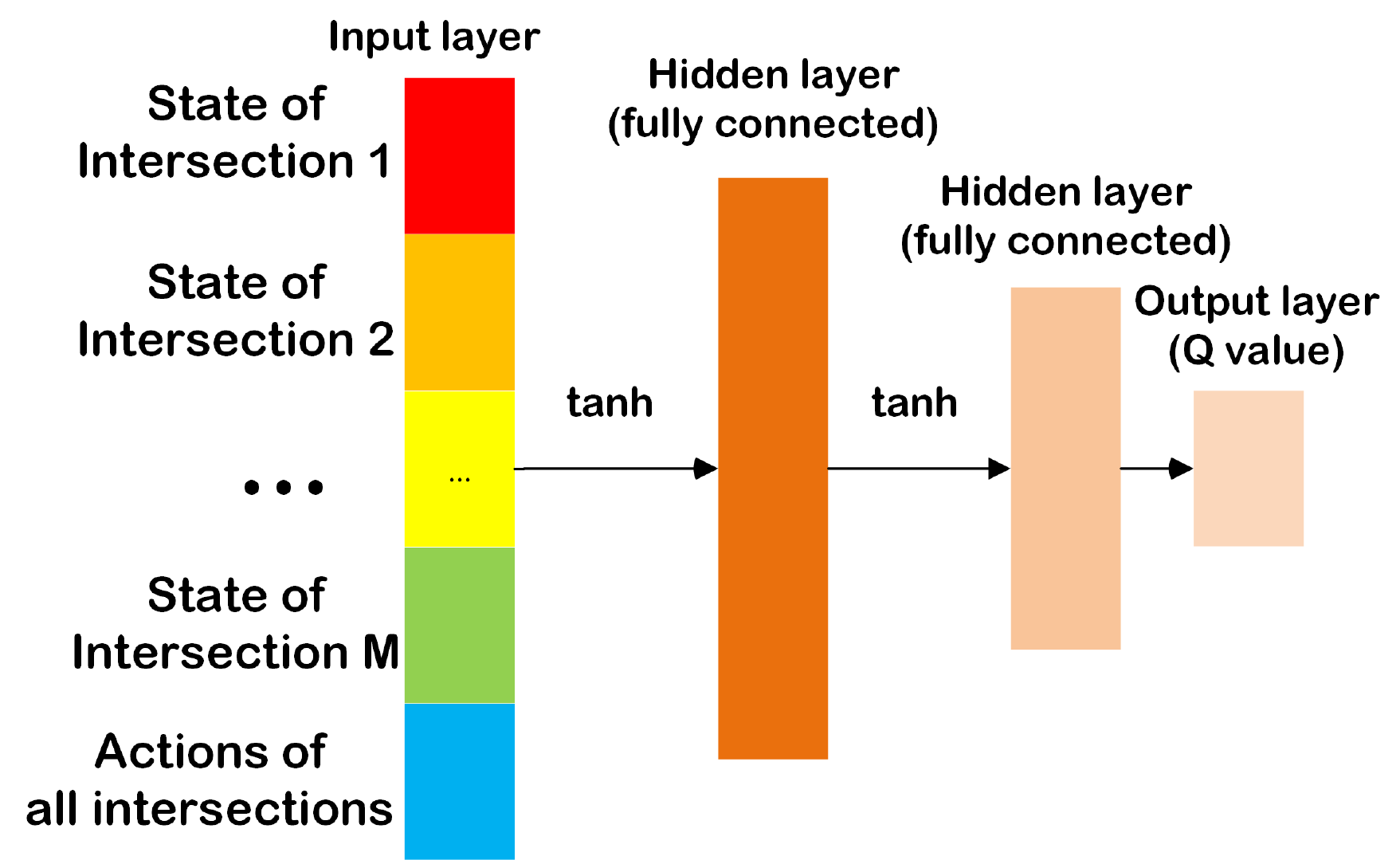}
     \centering
     (b)\includegraphics[scale=0.28]{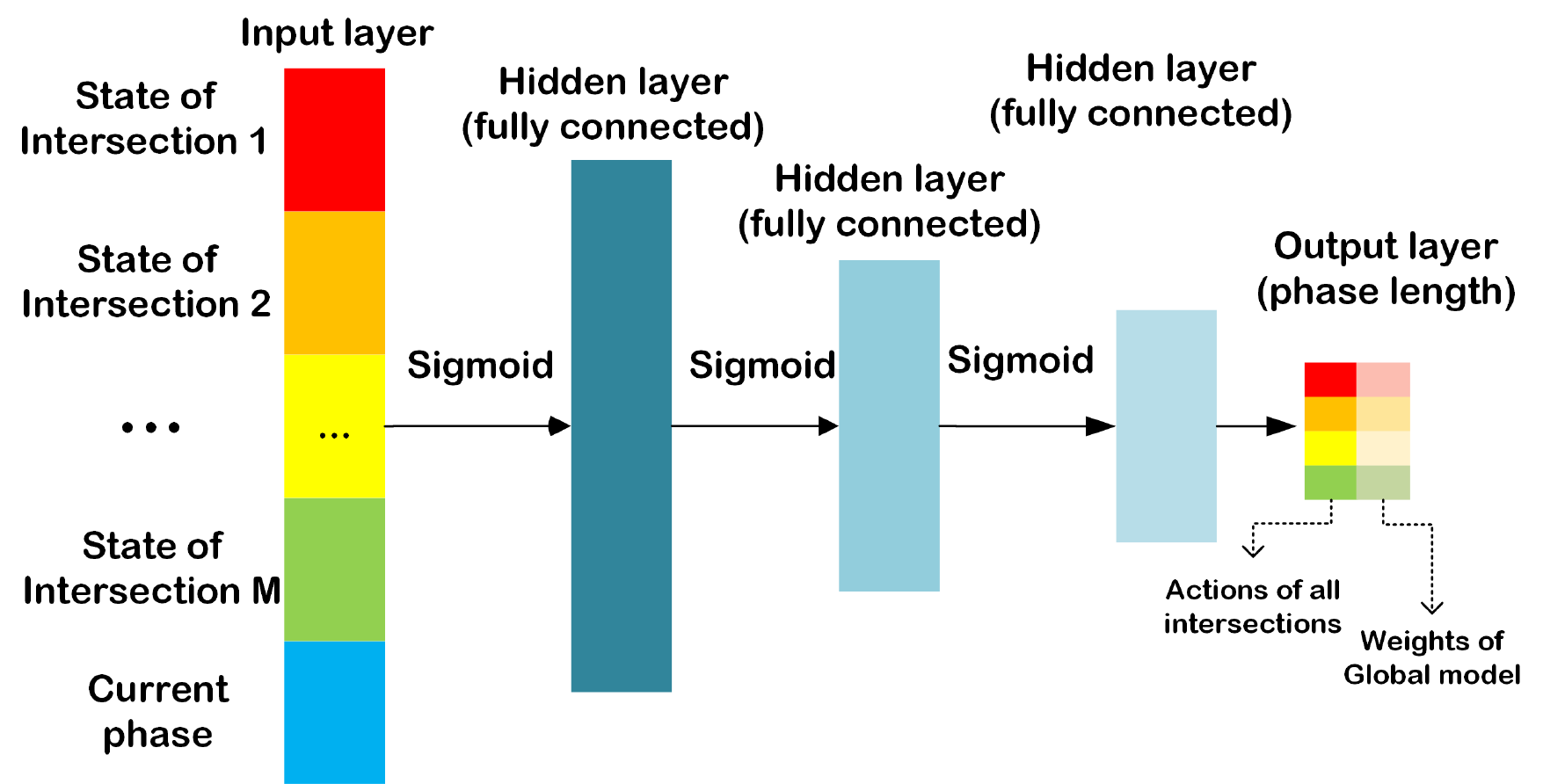}
     \caption{Architectures for global agent. (a) Global critic.(b) Global actor.}
     \label{fig:global_agent}
 \end{figure}

The original DDPG uses off-policy data and the Bellman equation \cite{bellmanequation} (see Eq.(\ref{eq:BellmanQ})) to learn the $Q$-function, and then to derive the policy.  It interleaves learning an approximator to find the best $Q^*(s,a)$ and also learning another approximator to decide the optimal action $a^*(s)$, and so in a way the action space is continuous.  The output of this DDPG is a continuous probability function to represent an action.  In this paper, an action corresponds to the seconds of green light. Although DDPG is off-policy, we  can mix the past data into the training set and thus make the distribution of the training set diverse by feeding current environment parameters to a traffic simulation software named TSIS~\cite{owen2000traffic} to provide on-policy data for RL training. 
 In a general DDPG, to make the agents have more opportunities to explore the environment, noise (random sampling) is added to the output action space during the training process, but will also makes the agent blindly explore the environments.  In the scenarios of traffic signal control,  most of SoTA methods adopted local agents to model different intersections.   During training, the same training mechanism ``adding noise to action model'' is used to make each agent explore the environment more.  However,``increasing the whole throughput'' is the same goal for all local agents.  The learning strategy ``adding noise to action model'' will decrease not only the effectiveness of learning but also the whole throughput since blinding exploration will make local agents choose conflict actions to other agents.  It means there should be a cooperation mechanism to be included to the DDPG mechanism among different local agents  to increase the final throughput during the learning process. The major novelty of this paper is to introduce a cooperative learning mechanism via a global agent to avoid local agents blindly exploring the environments so that the whole throughput and the learning effectiveness can be significantly improved.      


\subsection{Cooperative DPGG Network Architecture}
Most of policy-based RL methods~\cite{chu2019multi,nishi2018traffic,mousavi2017traffic} use only local agents to perform RL learning for traffic control.  The requirements of a local agent will easily produce conflicts to other agents and results in the divergence problem during optimization.  A cooperative DPGG architecture is designed in this paper, where a local agent controls each intersection and a global agent manages all intersections.  Details of this COMMA DPGG algorithm are described in {\bf Algorithm 1}.
 Although the DDPG method is off-policy, we use TSIS (Owen et al. 2000) to collect on-policy data for RL training. Details of the on-policy data collection process are described in the GOD (Generating On-policy Data) algorithm (see {\bf Algorithm 2}). With the set of on-policy data, the parameters of local and global agents are then updated by the {\bf LAU (Local Agent Updating)} algorithm and {\bf GAU(Global Agent Updating)} algorithm, respectively.  The global agent is involved only during the training stage to generate on-policy data.  Let $W_{G}^m$ represent the global agent's importance to the  $m$th intersection.  Then, the importance $W_{L}^m$ of the $m$th local agent will be 1-$W_{G}^m$, $i.e$.,  $W_{L}^m$ = 1- $W_{G}^m$.   
 For the $m$th intersection, the GOD algorithm  predicts the next actions by using the local agent and global agent via an epsilon greedy exploration scheme, respectively. The competition between the output seconds of the global agent and the local agent depends on $W_{G}^m$ and $W_{L}^m$.Then, the one with higher importance will be chosen to output seconds.  The output seconds are fed into TSIS (Owen et al. 2000) to generate on-policy data for RL training. To avoid the training being too biased to one side, we will have a penalty mechanism by the time-decayed method. Assuming that the model selects the global output for $t$ consecutive times, the global weight should be multiplied by $(0.95)^t$; that is, $W_{global}^m$=$W_{global}^m \times (0.95)^t$. This COMMA-DDPG  method  runs one hour of simulation with an epsilon greedy exploration scheme to collect on-policy data.  The set of on-policy data collected for training the $m$th local agent is denoted by $B^m$. The on-policy data set $\bf B$ for training the global agent is the union of all $B^m$, $i.e.$, $\bf B$= $(B^1,...,B^m , ..., B^M)$. In the following, details of each local agent and the global agent are described.
 
 \subsection{Generating On-policy Data}
As mentioned earlier, the major contribution of this paper is to add a global agent to trade off different local agents’requirements  and  cooperate  them  to  find  better  strategies  for traffic  signal  control.  Here we explain how global output and local output compete. The global output contains the number of seconds and weight $W_{G}^m$ which represent the global agent's importance to the  $m$th intersection. Then, the value $W_L^m$=(1-$W_{G}^m$) is the importance of the $m$th local agent to the global agent. The one with higher importance is chosen to output seconds.  To avoid the training being too biased to one side, we will have a penalty mechanism by the time-decayed method. Assuming that the model selects the global output for $t$ consecutive times, the global weight should be multiplied by $(0.95)^t$; that is, $W_{G}^m$=$W_{G}^m \times (0.95)^t$.  

Our method is based on MA-DDPG~\cite{gupta2017cooperative} with $M$ local agents, using the decentralized reinforcement learning method~\cite{matignon2007hysteretic}. The proposed COMMA-DDPG method adds a global agent, based on MA-DDPG, to control all intersections by using average stopped delay time of vehicle as the reward. Its actor output is not only the duration of the green light of each intersection, but also the weight $W_{G}^m$ relative to each local agent $m$. It is involved only during the training stage to generate on-policy data.   During the data generation process, a specific local agent is created at each intersection, using the clearance degree as the reward. As shown in Fig. \ref{fig:local_agent}(a), its critic’s input state will not only have information about its own intersection, but also takes the actions of other agents as its own state, so as to achieve information transmission between all the agents during the training process.

During the RL-based training process, before starting each epoch, we will perform a one-hour simulation to collect data (see {\bf Algorithm 2}) and store it in the replay buffer $\bf B$. In the process of interacting with the environment, we will add epsilon greedy and weight-decayed method to the selection of actions. In particular, the epsilon greedy method will gradually reduce epsilon from 0.9 to 0.1.

 \subsection{Local Agent}
In our scenario, a fixed duration of a traffic signal change cycle is assigned to each intersection.  In addition, there are 5 seconds prepared for the yellow light. Then, we only need to model the phase duration for the green light.  After that, the phase duration for the red light can be directly estimated.  At each intersection,  a DPGG-based architecture is constructed to model the local agent for traffic control.  To describe this local agent, some definitions are given as the following.

\begin{enumerate}
\item The duration of traffic phase ranges from $D_{min}$ to $D_{max}$ seconds.
\item Stopped vehicles are defined as those vehicles whose speeds are less than 3 $km/hr$.
\item The state at an intersection is defined by a vector in which each entry records the number of stopped vehicles of each lane at this intersection at the end of the green light, and current traffic signal phase.
\end {enumerate}
 The reward evaluating the quality of a state at an intersection is defined as the clearance degree of this state at this intersection, $i.e.$, the number of vehicles remaining in the intersection when the period of green light ends. There are two cases to give a reward to qualify a state; that is, (1) the green light ends but there is still traffic and (2) the green light is still but there is no traffic.   There is no reward or penalty for other cases.  Let $N_{m,t}$ denote the number of vehicles in the intersection $m$ at time $t$, and $N_{max}$ be the maximum traffic flow in the $m$th intersection.   This paper uses the clearance degree as a reward for qualifying the $m$th local agent.  When the green light ends and there is no traffic, a pre-defined max reward $R_{max}$ is assigned to the $m$th local agent.  If there is still traffic, a penalty proportional to  $N_{m,t}$ is given to this local agent.  More precisely, for Case 1, the reward $r_{m,t}^{local}$ for the $m$th intersection is defined as: \\
 {\bf Case 1}: If the green light ends but there is still traffic,
 \begin{equation}
     r_{m,t}^{local} = 
    \begin{cases}
      R_{max}, if\;\frac{{\;\;{N_{m,t}}}}{{{N_{max}}}} \le \frac{1}{N_{max}} \\
      { - R_{max} \times {N_{m,t}}/{N_{max}}}, \text{else}
    \end{cases}
  \end{equation}
For Case 2, if there is no traffic but a long period still remaining for the green light, various vehicles moving on another road should stop and wait until this green light turns off.  To avoid this case happening again, a penalty should be given to this local agent.  Let $g_{m,t}$ denote the remnant green light time (counted by seconds) when there is no traffic flow in the $m$th intersection at time step $t$, and $G_{max}$ the largest duration of green light.  Then, the reward function for Case 2 is defined as follows.\\
{\bf Case 2}: If there is no traffic but the green light is still on,
\begin{equation}
     r_{m,t}^{local} =
    \begin{cases}
      R_{max}, if\;\frac{{\;\;{g_{m,t}}}}{{{G_{max}}}} \le \frac{1}{G_{max}} \\
      { - R_{max} \times {g_{m,t}}/{G_{max}}}, \text{else}
    \end{cases}
  \end{equation}
Detailed architectures for local agents are shown in Fig. \ref{fig:local_agent}.  Fig. \ref{fig:local_agent}(a) shows the proposed local critic architecture.  Its inputs include the numbers of stopped vehicles at the end of the green light at each lane, the the remaining green light seconds, and current traffic signal phases of all intersections. Thus, the input dimension for each local critic network is $(2M+\sum _{m = 1}^{ {M} }{N_{lane}^m})$, where  $M$ denotes the number of intersections and ${N_{lane}^m}$ is the number of lanes in the $m$th intersection.  Then, a hyperbolic tangent function is used as an activation function to normalize all the input and output values.  There are two hidden fully-connected layers used to model the $Q$-value.  The output is the expected value of future return of doing the action at the state.

The architecture of the local actor network is shown in
Fig. \ref{fig:local_agent}(b).    Three inputs are used to model this network including the numbers of stopped vehicles at the end of the green light at each lane, and current traffic signal phases of all intersections. Thus, the input dimension for each local actor network is $(M+\sum _{m = 1}^{ {M} }{N_{lane}^m})$.

Let $\theta^Q_m$ and $\theta^{\mu}_m$ denote the sets of parameters of the $m$th local critic and actor networks, respectively. To train $\theta^Q_m$ and $\theta^{\mu}_m$, we sample a random minibatch of $N_b$ transitions $({\bm S}_{i},{\bm A}_{i},{\bm R}_{i},{\bm S}_{i+1})$ from $\bf B$, where  
\begin{enumerate}
\item each state ${\bm S}_{i}$ is an $M\times 1$ vector and contains the local states of all intersections.
\item each action  ${\bm A}_{i}$ is represented as an $M\times 1$ vector which contains the seconds of current phase of all intersections. 
\item reward ${\bm R_{i}}$ is an $M\times 1$ vector which contains the rewards obtained from each intersection after performing ${\bm A}_i$ at the state ${\bm S}_i$.  In addition, ${\bm R_{i} (m)}$ denotes the reward of the $m$th intersection after performing ${\bm A}_i$.
\end{enumerate}
Let $y_i^m$ denote the reward obtained from the $m$th target critic nework.  The loss function for updating $\theta^Q_m$ is then defined as follows:
\begin{equation}
    L_{critic}^m={1 \over {N_b}}\sum_{i=1}^{N_b}(y_{i}^m-Q({\bm S}_{i},{\bm A}_{i}|\theta_m^{Q}))^{2}.
\end{equation}
In addition, the loss function for updating $\theta^{\mu}_{m}$ is defined as
\begin{equation}
  L_{actor}^m=-{1 \over {N_b}}\sum_{i=1}^{N_b}Q({\bm S}_{i},\mu ({\bm S}_{i}|\theta_m^{\mu})|\theta_m^{Q}).
\end{equation}
With $\theta^{Q}_{m}$ and $\theta^{\mu}_{m}$, the parameters $\theta^{Q'}_{m}$ and $\theta^{\mu'}_{m}$ for the target network are updated as follows:
\begin{equation}
 \theta_m^{Q'} \leftarrow (1-\tau) \theta_m^{Q}+\tau\theta_m^{Q'},
\end{equation}
and 
\begin{equation}
 \theta_m^{\mu'} \leftarrow (1-\tau) \theta_m^{\mu}+\tau\theta_m^{\mu'}.
\end{equation}
The parameter $\tau$ is set to 0.8 for updating the target network. Details to update the parameters of local agents are described in {\bf Algorithm 3}.  
 \subsection{Global Agent}
To make the output action no longer blindly explore the environment, we introduces a global agent to explore the environment more precisely when outputting actions. The global agent controls  the  total  waiting time at all intersections. Fig. \ref{fig:global_agent} shows the detailed architectures of the global critic and actor networks, where (a) is the one of global critic network and (b) is for the global actor network. For the $m$th intersection, we use $V_m$ to denote the number of its total vehicles, and $T_{m,n}^{w,i}$ to be the waiting time of vehicle $n$ in  at the time step $i$. Then, the total waiting time across the whole site is used to define the global reward as follows:\\
\begin{equation}
  r_{i}^G=-\frac{1}{M}\sum_{m=1}^{M}\sum_{n=1}^{V_m}T_{m,n}^{w,i}.
\label{eq:GlobalReward}
\end{equation}
 Let $\theta^Q_G$ and $\theta^{\mu}_G$ denote the parameters of the global critic and actor networks, respectively. To train $\theta^Q_G$ and $\theta^{\mu}_G$, we sample a random minibatch of $N_b$ transitions $({\bm S}_{i},{\bm A}_{i},{\bm R}_{i},{\bm S}_{i+1})$ from $\bf B$.  Let $y_i^G$ denote the reward obtained from the global target critic network at the time step $i$.  Then, the loss function for updating $\theta^Q_G$ is defined as follows:
\begin{equation}
   L_{critic}^{G}=\frac{1}{N_b} \sum_{i}({y}_{i}^G-{Q_G}({\bm S}_{i},{\bm A}_{i}|\theta^{ Q}_G))^{2}.
\end{equation}
  The network architecture to calculate the value function $Q_G$ is shown in Fig. \ref{fig:global_agent}(a).  It is noticed that the output of this global critic network is a scalar value, $i.e.$, the predicted total waiting time across the whole site. To train $\theta^{\bm \mu}_{G}$, we use the loss function:
\begin{equation}
  L_{actor}^G=- {1 \over {N_b} } \sum_{i}{Q_G}({\bm S}_{i},{\bm{\mu_G}} ({\bm S}_{i}|\theta^{\bm{\mu}}_G)|\theta^{Q}_G).
\end{equation}
Fig. \ref{fig:global_agent}(b) shows the architecture to calculate $\bm {\mu}_G$.  In addition, the output of the global actor network is a vector which includes the actions of all intersections, and the weight $W_{G}^m$ which represent the global agent's importance to the  $m$th intersection.   
 All the local agents and global agent are modeled as a DDPG.  Details to update the global agent are described in {\bf Algorithm 4}.

\begin{algorithm}[]
\SetAlgoLined
Initialize critic network $Q(s,a|\theta^{Q})$ and actor network $\mu(s|\theta^{\mu})$ with random weights $\theta^{Q}$ and $\theta^{\mu}$.\\Initialize target network $Q'$ and $\mu '$\ with weights $\theta^{Q'} \leftarrow \theta^{Q}, \theta^{\mu '} \leftarrow \theta^{\mu}$ and also initialize replay buffer $R$.
  
\For{t=1, ... ,T}{
    Clean the replay buffer $\bf B$.\\
    /* $\bf {B}=(B_1,...,B_m ,...,B_M);$ */ \\
    /* $B^m$: on-policy data for the $m$th intersection */\\
    /* Generate on-policy data */ \\ 
    $\bf B$$=GOD(t)$;  \\
    \For{episode=1, ..., 400}{
        \For{m=1,..., M, Global}{
            \If{$m \neq Global$}{$LAU$($\bf B$,$m$);// Update local agents\\}
            \If{agent=Global}{$GAU$($\bf B$);// Update the global agent\\}
            }
        }
    }
\caption{COMMA-DDPG traffic signal control RL algorithm.}
\label{algo:1}
\end{algorithm}

\begin{algorithm}[]
\SetAlgoLined 
/* Run one hour of simulation
with noise  $\eta$*/\\
Input: \, \ $t$: timestamp\\ 
\, \, \, \, \, \, \ $\theta^{\mu}_m$: parameters for the $m$th actor network \\
\, \, \, \, \, \, \ $\theta^{\mu}_G$: parameters for the global actor network \\
Output: $\bf B$: on-policy data  \\

$\beta = 0.95^t$;// rate for time decline \\
\For{m=1, ... , M}{
 Get $W_{G}^m$ from the global actor network with the parameters  $\theta^{\mu}_G$;\\
 $W_{G}^m$= $\beta \times W_{G}^m$; $W_{L}^m$=1-$W_{G}^m$;\\
\For{l=1, ... ,3600}{

 /* $\epsilon$: the probability of choosing to explore */\\
/* $\eta_{m}$: noise for epsilon greedy exploration*/\\ 
$p=$ random(0,1);
 $\eta_{m}=
  \left\{
\begin{aligned}
0, if\  p \leq \epsilon,\\
random(-5,5), if\  p > \epsilon ;
\end{aligned}
\right.$\

    $a_{l}^m=
    \left\{
    \begin{aligned}
    \mu(s_{l}|\theta^{\mu}_{m})+\eta_{m}, \;\; \text{if} \; W_{L}^m>W_{G}^m,\\
    {\bm {\mu}}_{G}(s_{l}|\theta^{\mu}_{G})(m)+\eta_{m}, \;\; \text{if} \; W_{L}^m<W_{G}^m;
    \end{aligned}
\right.$

   Execute $a_{l}^m$ and observe $r_{l}^m, s_{l+1}^m$\;
  Store transition $(s_{l}^m,a_{l}^m,r_{l}^m,s_{l+1}^m)$ in $\bf{B_m}$\;
}

  }
 $\bf {B}=(B_1,...,B_m ,...,B_M);$\\
 Return(B);
\caption{ GOD (Generating On-policy Data)}
\label{algo:2} 
\end{algorithm}

\begin{algorithm}[]
\SetAlgoLined
Input: \\ 
\ \ \ $\bf B$: on-policy data; $m$: the $m$th agent  \\
\ \ \ $\theta_m^Q$: set of parameters for the local critic network; \\
\ \ \ $\theta_m^\mu$: set of parameters for the local actor network;  \\
\ \ \ ($\theta_m^{Q'}$,$\theta_m^{\mu '}$): sets of parameters for the target network;  \\
Output: \\ 
\ \ \ $\theta_m^Q$: new parameters for the $m$th critic network; \\
\ \ \ $\theta_m^\mu$: new parameters for the $m$th  actor network;  \\
\ \ \ ($\theta_m^{Q'}$,$\theta_m^{\mu '}$): new parameters for the target network;  \\
Sample a random minibatch of $N_b$ transitions $({\bm S}_{i},{\bm A}_{i},{\bm R}_{i},{\bm S}_{i+1})$ from $\bf B$\;
  Set $y_{i}^m={\bm R}_{i}(m)+\gamma Q'({\bm S}_{i+1}|\mu '({\bm S}_{i+1}|\theta_m^{\mu '})|\theta_m^{Q'})$\;
  Update the critic parameters $\theta_m^Q$ by minimizing the loss: $L_{critic}^m=\frac{1}{N_b} \sum_{i}(y_{i}^m-Q({\bm S}_{i},{\bm A}_{i}|\theta_m^{Q}))^{2}$\;
  Update the actor parameters $\theta_m^\mu$ by minimizing the loss: $L_{actor}^m=-\frac{1}{N_b}\sum_{i}Q({\bm S}_{i},\mu ({\bm S}_{i}|\theta_m^{\mu})|\theta_m^{Q})$\;
  Update the target network:
  $\theta_m^{Q'} \leftarrow (1-\tau) \theta_m^{Q}+\tau\theta_m^{Q'},\theta_m^{\mu'} \leftarrow (1-\tau) \theta_m^{\mu}+\tau\theta_m^{\mu'}$;
\caption{LAU (Local Agent Updating)}
\label{algo:3} 
\end{algorithm}

\begin{algorithm}[]
\SetAlgoLined
Sample a random minibatch of $N_b$ transition $({\bm S}_{i},{\bm A}_{i},{\bm R}_{i},{\bm S}_{i+1})$ from $\bf B$\;
  Calculate $r_i^G$ from Eq.(\ref{eq:GlobalReward}).\\
  Set ${y}^G_{i}={r}_{i}^G+\gamma {Q_G'}({\bm S}_{i+1}|{\bm {\mu}}_G'({\bm S}_{i+1}|\theta^{\bm{\mu}'}_G)|\theta^{Q'}_G)$\;
  Update the critic parameter $\theta^{ Q}_G$ by minimizing the loss: $L_{critic}^{G}=\frac{1}{N_b} \sum_{i}({y}_{i}^G-{Q_G}({\bm S}_{i},{\bm A}_{i}|\theta^{ Q}_G))^{2}$\;
  Update the actor parameter $\theta^{\bm{\mu}}_G$ by minimizing the loss: $L_{actor}^G=-$  $1 \over {N_b}$ $\sum_{i}{Q_G}({\bm S}_{i},{\bm{\mu_G}} ({\bm S}_{i}|\theta^{\bm{\mu}}_G)|\theta^{ Q}_G) $\;
  Update the target networks:
  $\theta^{Q'}_G \leftarrow (1-\tau) \theta^{Q}_G+\tau\theta^{Q'}_G,\theta^{\bm{\mu'}}_G \leftarrow (1-\tau) \theta^{\bm{\mu}}_G+\tau \theta^{\bm {\mu'}}_G$
\caption{GAU(Global Agent Updating)}
\label{algo:4} 
\end{algorithm}

\section{Experimental Results}
\subsection{Environment Setup}
It is difficult to test and evaluate traffic signal control strategies in the
real world due to high cost and intensive labor.  Simulation is a
useful alternative method before actual implementation for most SoTA methods~\cite{wei2021recent}. To build the simulation data, real data were collected from five real intersections in a city in Asia among half a year. The simulation with real traffic flow at the intersections was performed based on the traffic simulation software  TSIS~\cite{owen2000traffic}. Through TSIS, we can control the behaviors of each traffic light with a plug-in program and use it as the simulation software needed for performance evaluation. 

\begin{table}[t]
\centering
\caption{ Comparisons of waiting time and speed between COMMA-DDPG and other traditional RL methods.
\vspace{-0.1cm}}
\begin{tabular}{lcccc}
Method     & Waiting Time & Average Speed     \\\hline
Fixed      & 750628       & 19                \\
IntelliLight~\cite{wei2018intellilight} & xxx       & xx                \\
MADDPG~\cite{gupta2017cooperative}    & 420561       & 20                \\
TD3~\cite{fujimoto2018addressing}      & 716481       & 20                \\
PPO~\cite{schulman2017proximal} & 873585       & 12                \\
CGRL~\cite{VanDerPol16LICMAS} & xxx       & xx                \\
Presslight~\cite{wei2019presslight} & xxx       & xx                \\ 
CoLight~\cite{wei2019colight}       & xxx       & xx                \\\hline
COMMA-DDPG w/o Global Agent        & xxx       & xxx                \\
COMMA-DDPG with Global Agent& 269747       & 43                \\\hline
\end{tabular}
\end{table}

\begin{figure}
    \centering
    \includegraphics[scale = 0.2]{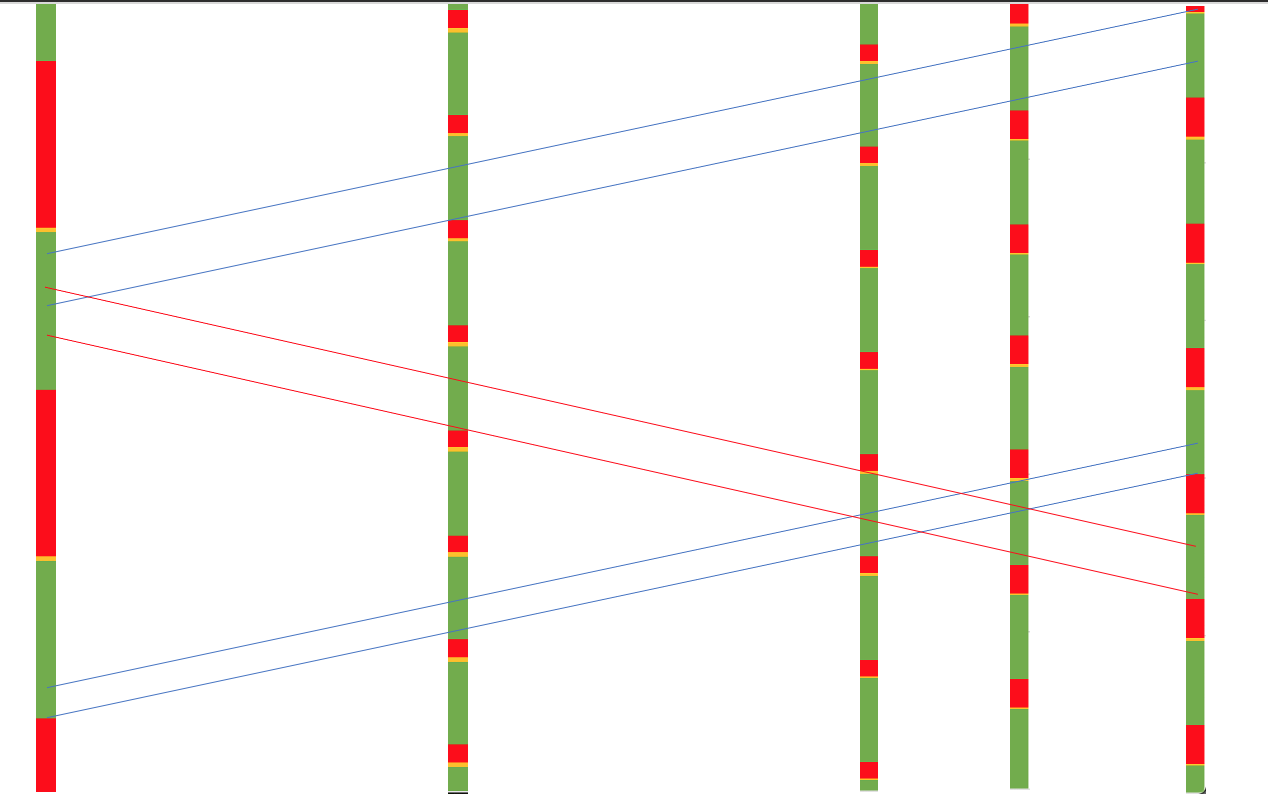}
    \caption{Time and space diagram.}
    \label{fig:time_and_space}
\vspace{-0.4cm}
\end{figure}

\begin{table}[t]
\centering
\caption{ Comparisons of throughput between COMMA-DDPG and other traditional RL methods.
\vspace{-0.1cm}}
\begin{tabular}{lccccc}
Method     & \begin{tabular}[c]{@{}c@{}}I1\end{tabular} & \begin{tabular}[c]{@{}c@{}}I2\end{tabular} & \begin{tabular}[c]{@{}c@{}}I3\end{tabular} & \begin{tabular}[c]{@{}c@{}}I4\end{tabular} & \begin{tabular}[c]{@{}c@{}}I5\end{tabular} \\\hline
Fixed  & 1530  & 1560  & 1996  & 2288 & 2291\\
MADDPG   & 1782  & 1819 & 2098 & 1896 & 2400\\
TD3  & 1370 & 1394 & 1787 & 2070 & 2147 \\
PPO & 979 & 957 & 1206 & 1517 & 1619\\\hline
COMMA-DDPG & 2225 & 2310 & 2784 & 3052 & 2868               
\end{tabular}
\vspace{-0.5cm}
\end{table}

\subsection{Results}
To train our method, the fixed-time control model was first used to pretrain our COMMA-DDPG model for speeding up the efficiency of training. The base line is a fixed strategy.  Three SoTA methods were compared in this paper; that is,  CGRL~\cite{VanDerPol16LICMAS}, MADDPG~\cite{gupta2017cooperative}, TD3~\cite{fujimoto2018addressing}, PPO~\cite{schulman2017proximal}, Presslight~\cite{wei2019presslight}, IntelliLight~\cite{wei2018intellilight}, CoLight~\cite{wei2019colight}.  Table 1 shows the comparisons of waiting time and average speed of vehicles among different methods. Clearly, our method performs better than the fixed-time model and other SoTA methods. Table 2 shows the comparisons of throughput between COMMA-DDPG and other methods at different intersections. Due to the control of the global agent, our method performs much better than other methods.  In algorithm 3 and 4, there is a soft update $\tau$ to update the network parameters. Table 3 shows the effects of change of $\tau$ to the training result in different situations. It means better performance can be gained if the model is not changed frequently.  DDPG is an off-policy method.  In Algorithm 2, an on-policy data collection method is proposed to train the agents.  The results in Table 4 illustrate the theory we mentioned earlier, that is, using on policy training can achieve better results.
\begin{table}[t]
\centering
\caption{ Comparisons of $\tau$ between random sample and fixed.
\vspace{-0.1cm}}
\begin{tabular}{lcc}
Updating Ratio $\tau$                                                           & Waiting Time & Average Speed \\\hline
\begin{tabular}[c]{@{}l@{}} $random(0,1)$\end{tabular}   & 436538       & 36            \\
\begin{tabular}[c]{@{}l@{}} $random(0.8,1)$\end{tabular} & 626936       & 27            \\
\begin{tabular}[c]{@{}l@{}} $random(0.9,1)$\end{tabular} & 675703       & 26         \\\hline
\begin{tabular}[c]{@{}l@{}} $\tau$=0.995\end{tabular}   & 269747       & 43            \\
\end{tabular}
\vspace{-0.3cm}
\end{table}

\begin{table}[]
\centering
\caption{ Comparisons between on-policy and off-policy training.
\vspace{-0.1cm}}
\begin{tabular}{lcc}
Method & Waiting Time & Average Speed \\\hline
\begin{tabular}[c]{@{}l@{}} on-policy\end{tabular}  & 269747       & 43            \\
\begin{tabular}[c]{@{}l@{}} off-policy \end{tabular} & 275868       & 29           
\end{tabular}
\end{table}

Fig~\ref{fig:time_and_space} shows that the distance between the two parallel lines is the Green Band, and its slope represents the driving speed. It shows that even during peak working hours, we can continue to pass through all intersections in the system without being hindered by red lights if we use the average speed as the designed continuous speed to drive.
\section{Conclusions}
This paper proposed a novel cooperative RL architecture to handle cooperation problems by adding a global agent. Since the global agent knows all the intersection information, it can guide the local agent to make better actions in the training process, so that the local agent does not use random noise to randomly explore the environment, but has a directional direction. explore. Since RL training requires a large amount of data, we hope to add it to RL through data augmentation in the future, so that training can be more efficient.


\section{Appendix for Convergence Proof}

In this section, we will prove that value function in our method will actually converge.
\begin{definition}
A metric space $<M,d>$ is complete (or Cauchy) if and only if all Cauchy sequences in $M$ will converge to $M$. In other words, in a complete metric space, for any point sequence $a_{1},a_{2}, \cdots \in M$, if the sequence is Cauchy, then the sequence converges to $M$:
\[ \lim_{n \rightarrow \infty}a_{n} \in M. \]
\end{definition}

\begin{definition}
Let (X,d) be a complete metric space. Then, a map T : X $\rightarrow$ X is called a contraction mapping on X if there exists q $\in [0, 1)$ such that $d(T(x),T(y))<qd(x,y)$, $\forall x,y \in X$.
\end{definition}
\begin{theorem}[Banach fixed-point theorem]
Let (X,d) be a non-empty complete metric space with a contraction mapping T : X $\rightarrow$ X. Then T admits a unique fixed-point $x^{*}$ in X. i.e. $T(x^{*})=x^{*}.$
\end{theorem}

\begin{theorem}[Gershgorin circle theorem]
Let A be a complex $n\times n$ matrix, with entries $a_{ij}$. For $i \in {1,2,...,n}$, let $R_{i}$ be the sum of the absolute of values of the non-diagonal entries in the $i^{th}$ row:
$$R_{i}=\sum_{j=0,j\neq i}^{n}|a_{ij}|.$$
Let $D(a_{ii},R_{i})\subseteq \mathbb{C} $ be a closed disc centered at $a_{ii}$ with radius $R_{i}$, and every eigenvalue of ${\displaystyle A}$ lies within at least one of the Gershgorin discs ${\displaystyle D(a_{ii},R_{i}).}$
\end{theorem}
\begin{lemma}
We claim that the value function of RL can actually converge, and we also apply it to traffic control.
\end{lemma}
\begin{proof}
The value function is to calculate the value of each state, which is defined as follows:
\begin{equation}
\begin{array}{l}
{V^\pi }(s) = \sum\limits_a \pi  (a|s)\sum\limits_{s',r} p (s',r|s,a)[r + \gamma {V^\pi }(s')]\\
 = \sum\limits_a \pi  (a|s)\sum\limits_{s',r} p (s',r|s,a)r\\
{\rm{ }} + \sum\limits_a \pi  (a|s)\sum\limits_{s',r} p (s',r|s,a)[\gamma {V^\pi }(s')].
\end{array}
\end{equation}
Since the immediate reward is determined, it can be regarded as a constant term relative to the second term. Assuming that the state is finite, we express the state value function in matrix form below.
Set the state set $S=\{S_{0},S_{1},\cdots,S_{n}\}$, $V^{\pi}=\{ V^{\pi}(s_{0}), V^{\pi}(s_{1}), \cdots , V^{\pi}(s_{n}) \}^{T}$, and the transition matrix is
\begin{equation}
    {P^\pi } = \left( {\begin{array}{*{20}{c}}
0&{P^\pi _{0,1}}& \cdots &{P^\pi _{0,n}}\\
{P^\pi _{1,0}}&0& \cdots &{P^\pi _{1,n}}\\
 \cdots & \cdots & \cdots & \cdots \\
{P^\pi _{n,0}}&{P^\pi _{n,1}}& \cdots &0
\end{array}} \right),
\end{equation}
where $P^\pi _{i,j} = \sum\limits_a {\pi (a|{s_i})p({s_j},r|{s_i},a)}$.  The constant term is expressed as $R^{\pi}=\{ R_{0}, R_{1}, \cdots, R_{n}\}^{T}$. Then we can rewrite the state-value function as:
\begin{equation}
V^{\pi}=R^{\pi}+\lambda P^{\pi}V^{\pi}.
\end{equation}
Above we define the state value function vector as $V^{\pi}=\{ V^{\pi}(s_{0}), V^{\pi}(s_{1}),\cdots, V^{\pi}(s_{n})\}^{T}$, which belongs to the value function space $V$. We consider $V$ to be an n-dimensional vector full space, and define the metric of this space is the infinite norm. It means:
\begin{equation}
    d(u,v)=\parallel u-v \parallel_{\infty}=\max_{s \in S}|u(s)-v(s)|,\forall u,v \in V
\end{equation}
Since $<V,d>$ is the full space of vectors, $V$ is a complete metric space. Then, the iteration result of the state value function is $u_{new}=T^{\pi}(u)=R^{\pi}+\lambda P^{\pi}u$.
We can show that it is a contraction mapping.
\begin{equation}
    \begin{aligned}
    d(T^{\pi}(u),T^{\pi}(v))&=\parallel (R^{\pi}+\lambda P^{\pi}u)-(R^{\pi}+\lambda P^{\pi}v)\parallel_{\infty}\\
    &=\parallel \lambda P^{\pi}(u-v) \parallel_{\infty}\\
    &\le \parallel \lambda P^{\pi}\parallel u-v \parallel_{\infty}\parallel _{\infty}.
    \end{aligned}
\end{equation}
From Theorem 2, we can show that every eigenvalue of $P^{\pi}$ is in the disc centered at $(0,0)$ with radius 1. That is, the maximum absolute value of eigenvalue will be less than 1.
\begin{equation}
    \begin{aligned}
    d(T^{\pi}(u),T^{\pi}(v))&\le \parallel\lambda P^{\pi}\parallel u-v \parallel_{\infty}\parallel_{\infty}\\
    &\le \lambda \parallel u-v \parallel_{\infty}\\
    &=\lambda d(u,v).
    \end{aligned}
\end{equation}
From the Theorem 1, Eq.(2) converges to only $V^{\pi}$.
\end{proof}

\bibliographystyle{IEEEtrans}
\bibliography{ITS_TrafficSignalControl}

%








\end{document}